%% file: neurips_2026.tex
\documentclass{article}
\usepackage{wrapfig}

\usepackage{adjustbox}
\usepackage{booktabs}
\usepackage{amsthm}
\usepackage[most]{tcolorbox}
\usepackage[table]{xcolor}
\usepackage{subcaption}
\definecolor{promptbrown}{RGB}{150,70,20}
\definecolor{promptgreen}{RGB}{0,110,45}
\definecolor{promptpurple}{RGB}{120,45,170}
\definecolor{promptblue}{RGB}{20,40,150}
\newtheorem{proposition}{Proposition}
\usepackage[main,final]{neurips_2026}

\usepackage[utf8]{inputenc} 
\usepackage[T1]{fontenc}    
\usepackage{hyperref}       
\usepackage{url}            
\usepackage{booktabs}       
\usepackage{amsfonts}       
\usepackage{nicefrac}       
\usepackage{microtype}      
\usepackage{xcolor}         
\usepackage{amsmath} 
\usepackage{graphicx}
\title{\textsc{RICE-PO}: Turning Retrieval Interactions into Credit Signals for Reasoning Agents}

\author{
Mingchen Li\textsuperscript{\normalfont 1,5}\thanks{%
indicates co-author
\hspace{0.8em}\textbar\hspace{0.8em}
The code is available at
\href{https://github.com/ToneLi/RICE_PO}{GitHub}.%
},
Hansi Zeng\textsuperscript{\normalfont 1}\footnotemark[1],
Zhuo Qian\textsuperscript{\normalfont 2}\footnotemark[1],
Jiatan Huang\textsuperscript{\normalfont 3},
Sunjae Kwon\textsuperscript{\normalfont 1} \\
\textbf{Hamed Zamani}\textsuperscript{\normalfont 1},
\textbf{Hong Yu}\textsuperscript{\normalfont 1,4,5}\\[6pt]
\textsuperscript{1}University of Massachusetts Amherst \qquad
\textsuperscript{2}Texas Tech University \\[3pt]
\textsuperscript{3}University of Connecticut \qquad
\textsuperscript{4}University of Massachusetts Lowell \\[3pt]
\textsuperscript{5}VA Bedford Health Care
}

\begin{document}



\maketitle

\input{1_abs}

\input{2_introduction}
\input{2_related_work}
\input{3_preliminaries}
\input{3_method}
\input{4_experiment}

\input{6_conclution}

\bibliography{reference}
\bibliographystyle{plain}

\appendix
\section{Technical appendices and supplementary material}

\subsection{Implementation}

\begin{table*}[t]
\centering
\caption{Statistics of the BRIGHT and BEIR benchmark splits used in our experiments.}
\label{tab:rl_dataset_stats}
\footnotesize
\setlength{\tabcolsep}{1.5pt}
\renewcommand{\arraystretch}{1.15}

\begin{adjustbox}{max width=\textwidth}
\begin{tabular}{@{}l*{19}{c}@{}}
\toprule
& \multicolumn{13}{c}{\textbf{BRIGHT}}
& \multicolumn{6}{c}{\textbf{BEIR}} \\
\cmidrule(lr){2-14}
\cmidrule(lr){15-20}

Split
& Bio. & Earth. & Econ. & Psy. & Rob.
& Stack. & Sus. & Leet. & Pony & AoPS
& TheoQ. & TheoT. & Total
& DBP. & FiQA & SciFact & TREC & Touch\'e & Total \\
\midrule

Train
& 414 & 404 & 412 & 424 & 275
& 418 & 426 & 420 & 94 & 384
& 0 & 340 & 4011
& 23 & 3183 & 802 & 0 & 0 & 4008 \\

Test
& 103 & 116 & 103 & 101 & 101
& 117 & 108 & 142 & 112 & 111
& 194 & 76 & 1384
& 400 & 648 & 300 & 50 & 49 & 1447 \\

\midrule

Total
& 517 & 520 & 515 & 525 & 376
& 535 & 534 & 562 & 206 & 495
& 194 & 416 & 5395
& 423 & 3831 & 1102 & 50 & 49 & 5455 \\

\bottomrule
\end{tabular}
\end{adjustbox}
\end{table*}

 Since BRIGHT and BEIR are primarily designed for evaluation, we construct BRIGHT-style training data based on the dataset released by ReasonIR. ReasonIR builds query-document pairs for reasoning-intensive information retrieval using prompt engineering and human annotation, covering multiple domains aligned with BRIGHT. We sample 4,011 examples for training and 1,384 examples for testing. For the BEIR datasets, we construct the training and validation splits from the available training data of DBPedia-Entity, FiQA-2018, and SciFact. Specifically, we sample 4,008 examples for training and 1,447 examples for testing. Detailed task-level statistics for both BRIGHT and BEIR are provided in Appendix~\ref{tab:rl_dataset_stats}.  We use DeepSeek-R1-Distill-Qwen-1.5B\footnote{\url{https://huggingface.co/deepseek-ai/DeepSeek-R1-Distill-Qwen-1.5B}}, Qwen3-4B-Thinking-2507\footnote{\url{https://huggingface.co/Qwen/Qwen3-4B-Thinking-2507}}, and Qwen2.5-3B-Instruct\footnote{\url{https://huggingface.co/Qwen/Qwen2.5-3B-Instruct}} as our base models. 
For policy optimization, we set the actor KL coefficient to \(0.001\). 
For our credit-estimation module, we select the top \(K=4\) summary actions according to entropy and sample \(5\) Monte Carlo branches for each selected anchor to estimate reasoning-to-summary influence and residual effects. 
The influence threshold and residual-effect threshold are set to \(0.05\) and \(0.03\), respectively. 
We train with a batch size of \(128\), an actor learning rate of \(1\times10^{-6}\), and a maximum rollout depth of \(5\) retrieval steps.  Following Diver, we use Diver-4B with BM25 as the retriever for all RL-based methods in this paper.

\subsection{Prompt Template for Multi-step Retrieval Agent}

Figure~\ref{fig:prompt_template}  shows the prompt template for the multi-step retrieval agent.

\definecolor{promptdarkgreen}{HTML}{315B4B}
\definecolor{promptlightgreen}{HTML}{F1F8F4}

\definecolor{sectionpurple}{HTML}{7A3EB1}
\definecolor{sectionred}{HTML}{B23A48}
\definecolor{sectionblue}{HTML}{245A85}
\definecolor{sectionorange}{HTML}{B06A00}

\definecolor{thinkgreen}{HTML}{1B8A5A}
\definecolor{summarypurple}{HTML}{8A3FFC}

\newcommand{\important}[1]{\textbf{\textcolor{sectionred}{#1}}}
\newcommand{\diversity}[1]{\textbf{\textcolor{sectionpurple}{#1}}}
\newcommand{\strictfmt}[1]{\textbf{\textcolor{sectionblue}{#1}}}
\newcommand{\detailsreq}[1]{\textbf{\textcolor{sectionorange}{#1}}}

\newcommand{\thinktag}[1]{\textcolor{thinkgreen}{\texttt{#1}}}
\newcommand{\summarytag}[1]{\textcolor{summarypurple}{\texttt{#1}}}

\begin{figure*}[t]
\centering
\begin{tcolorbox}[
    enhanced,
    width=\textwidth,
    colback=promptlightgreen,
    colframe=promptdarkgreen,
    colbacktitle=promptdarkgreen,
    coltitle=white,
    title={Prompt Template for Multi-step Retrieval Agent},
    fonttitle=\large\bfseries,
    boxrule=0.7pt,
    arc=0pt,
    outer arc=0pt,
    boxsep=0pt,
    left=4mm,
    right=4mm,
    top=3mm,
    bottom=3mm,
    toptitle=1.5mm,
    bottomtitle=1.5mm
]

\ttfamily
\footnotesize
\raggedright
\setlength{\parindent}{0pt}
\setlength{\parskip}{4pt}

You are a professional multi-step retrieval agent.
Your role is to iteratively refine a query through
thinking and summarization, and use an external
retriever (not shown to you) to gather increasingly
relevant documents.

You will be given: a user query wrapped in
\texttt{<query> ... </query>}, and an initial set of
retrieved documents wrapped in
\texttt{<information> ... </information>}.

At each step, produce a refined summary to request
another retrieval round.

\important{IMPORTANT:}
The \summarytag{<summary>} you produce will be used directly
as the query for the next retrieval round.
It guides the system to fetch a new set of documents
in \texttt{<tool\_response> ... </tool\_response>}.
Although your summaries help obtain better evidence,
they do NOT replace or modify the original user query.

\diversity{DIVERSITY CONSTRAINTS (CRITICAL):}
For each round, choose ONE primary retrieval angle
and center the summary on it.
Candidate angles include: definition, mechanism,
evidence, counterexample, boundary condition,
implementation detail, evaluation metric, temporal
context, and domain-specific terminology.

Do not reuse the same opening sentence pattern
across rounds. Avoid generic filler.
Prefer concrete entities, key terms, constraints,
and disambiguation cues.
Keep lexical overlap with previous summary low;
rewrite with different wording when possible.
If multiple hypotheses exist, prioritize one
and add 1--2 alternatives as optional hooks.

\strictfmt{STRICT OUTPUT FORMAT:}
At every step, your output MUST be in the following
format:

\thinktag{<think> ... </think>}

\summarytag{<summary> ... </summary>}

After you output \summarytag{<summary>}, the system will
automatically retrieve new documents and provide
them to you in the next turn.
Do NOT generate any placeholder or tag after
\summarytag{</summary>}.

\detailsreq{DETAILS \& REQUIREMENTS:}

\thinktag{<think> ... </think>}
Evaluate whether the current documents address
the original user query.
Describe what parts of the documents satisfy
the query and what is still missing.
MAXIMUM 500 words.

\summarytag{<summary> ... </summary>}
Produce a retrieval-oriented summary query.
500--1000 words.
A refined, retrieval-friendly description
of missing information.
This summary will become the query for the next
retrieval round.
Combine the original query with only the latest
round context.
Include 4--8 concrete keywords or entities
that increase retrieval specificity.
Include at least one exclusion or disambiguation
phrase, for example: not X, focus on Y.

\end{tcolorbox}
\caption{Prompt template for the multi-step retrieval agent.}
\label{fig:prompt_template}
\end{figure*}

\subsection{Proof of Proposition~\ref{prop:summary_credit}}

\begin{proof}
Let $h_t$ be the history before step $t$, and let $z_t$ be
the reasoning action that induces summary $s_t$.
Define the reasoning-level advantage as
$A_t^{\mathrm{think}}
= \mathbb{E}[R \mid h_t,z_t]
- \mathbb{E}_{z_t'\sim\pi(\cdot\mid h_t)}
[R \mid h_t,z_t']$.
Assume that the final reward can be decomposed as
$R = r(s_t) + \epsilon_t$,
where $r(s_t)$ is the immediate retrieval reward of the
induced summary and $\epsilon_t$ denotes the remaining
future effect. Define the summary-level advantage as
$A_t^{\mathrm{sum}}
= r(s_t)
- \mathbb{E}_{z_t'\sim\pi(\cdot\mid h_t)}
[r(s_t')]$.

Then,
\begin{equation}
A_t^{\mathrm{think}}
=
A_t^{\mathrm{sum}}
+
\Delta_{\epsilon},
\end{equation}
where
$\Delta_{\epsilon}
= \mathbb{E}[\epsilon_t \mid h_t,z_t]
- \mathbb{E}_{z_t'\sim\pi(\cdot\mid h_t)}
[\epsilon_t' \mid h_t,z_t']$.
Therefore, if the residual future effect is bounded,
$|\Delta_{\epsilon}| \leq \delta$,
then the summary-level advantage approximates the
reasoning-level advantage with error at most $\delta$:
\begin{equation}
\left|
A_t^{\mathrm{think}}
-
A_t^{\mathrm{sum}}
\right|
\leq \delta.
\end{equation}

Moreover, if the induced summary has a non-trivial advantage,
$|A_t^{\mathrm{sum}}| \geq \tau$,
then the summary signal is strong enough to provide
meaningful credit to the reasoning action.
Thus, when both conditions hold, $A_t^{\mathrm{sum}}$
can be used as a reliable proxy for assigning credit
to the reasoning tokens that produced $s_t$.
\end{proof}

\subsection{Effect of the Entropy-Selected Anchor Budget}
\label{app:anchor_budget}

We investigate the sensitivity of RICE-PO to the number of
entropy-selected summary anchors, which directly determines
the coverage and computational cost of local counterfactual
credit estimation. To isolate the effect of the anchor budget, we
vary $K \in \{1,2,4,8\}$ while keeping the policy model,
retrieval configuration, maximum interaction depth, gating
thresholds, and Monte Carlo branching budget fixed.
Specifically, we use DeepSeek-R1-Distill-Qwen-1.5B as the
policy model and retain the same BM25 and
Diver-Retriever-4B retrieval configuration across all
settings. We fix the local Monte Carlo branching budget
at five branches per selected anchor. We evaluate retrieval
performance using NDCG@10 across the 12 BRIGHT tasks
and report the corresponding wall-clock time.

\begin{table*}[t]
\centering
\caption{Sensitivity to the number of entropy-selected
summary anchors ($K$) across 12 BRIGHT tasks. We report
NDCG@10 scores ($\times 100$), their average, and the
wall-clock time. The policy model, retrieval configuration,
and Monte Carlo branching budget remain fixed.}
\label{tab:anchor_budget}
\resizebox{\textwidth}{!}{
\begin{tabular}{lcccccccccccccc}
\toprule
$K$ & Bio. & Earth. & Econ. & Psy. & Rob. &
Stack. & Sus. & Leet. & Pony & AoPS &
TheoQ. & TheoT. & Avg. & Time (h) \\
\midrule
1
& 27.06 & 29.70 & 17.16 & 24.24 & 16.70
& 19.97 & 16.99 & 29.19 & 2.60 & 6.77
& 30.90 & 31.48 & 21.06 & 7.21 \\

2
& 40.80 & 41.91 & 19.77 & 28.48 & 17.38
& 23.26 & 19.74 & 30.30 & 9.60 & 10.24
& 34.29 & 39.30 & 26.26 & 10.02 \\

4
& 40.17 & 43.79 & 18.35 & 32.13 & 17.47
& 22.61 & 22.54 & 34.13 & 10.79 & 10.77
& 40.06 & 37.00 & \textbf{27.49} & 11.36 \\

8
& 29.94 & 36.23 & 19.31 & 21.30 & 14.40
& 20.00 & 17.20 & 21.60 & 8.22 & 9.23
& 26.01 & 31.74 & 21.27 & 13.03 \\
\bottomrule
\end{tabular}
}
\end{table*}

\paragraph{Retrieval performance.}
As shown in Table~\ref{tab:anchor_budget}, increasing
the anchor budget from $K=1$ to $K=2$ improves average
NDCG@10 from 21.06 to 26.26, a gain of 5.20 points.
Further increasing $K$ to 4 yields the best average
performance of 27.49, providing an additional gain of
1.23 points. However, increasing the budget to $K=8$
reduces average NDCG@10 to 21.27. These results indicate
that selecting more uncertain summary actions does not
necessarily improve credit estimation. Beyond a moderate
anchor budget, additional selected actions may introduce
less informative or noisy local credit signals.
\paragraph{Performance--compute trade-off.}
Increasing $K$ also raises computational cost, as more
selected anchors require counterfactual evaluation.
The wall-clock time increases from 7.21 hours at $K=1$
to 10.02 hours at $K=2$, 11.36 hours at $K=4$, and
13.03 hours at $K=8$. Relative to $K=4$, reducing the
anchor budget to $K=2$ saves 1.34 hours while decreasing
average NDCG@10 by only 1.23 points. In contrast,
reducing $K$ to 1 causes a larger performance degradation
of 6.43 points. Increasing $K$ from 4 to 8 incurs an
additional 1.67 hours while decreasing average NDCG@10
by 6.22 points.
Overall, $K=4$ achieves the best performance among the
evaluated configurations, whereas $K=2$ provides a
lower-cost alternative with a modest performance reduction.
The degradation at $K=8$ highlights the diminishing
returns of expanding anchor coverage and motivates a
moderate uncertainty-guided selection budget.

\subsection{Effect of the Counterfactual Branch Budget}
\label{app:branch_budget}

\begin{table*}[t]
\centering
\caption{Effect of the counterfactual branch budget $B$
on NDCG@10 ($\times 100$) across 12 BRIGHT tasks.}
\label{tab:branch_budget}
\resizebox{\textwidth}{!}{
\begin{tabular}{lccccccccccccc}
\toprule
$B$ & Bio. & Earth. & Econ. & Psy. & Rob. &
Stack. & Sus. & Leet. & Pony & AoPS &
TheoQ. & TheoT. & Avg. \\
\midrule
2
& 34.41 & 43.33 & 19.10 & 25.69 & 15.36
& 21.22 & 15.59 & 25.48 & 6.56 & 10.42
& 28.45 & 40.37 & 23.83 \\
5
& 40.17 & 43.79 & 18.35 & 32.13 & 17.47
& 22.61 & 22.54 & 34.13 & 10.79 & 10.77
& 40.06 & 37.00 & \textbf{27.49} \\
\bottomrule
\end{tabular}
}
\end{table*}

We examine the effect of the counterfactual branching budget $B$ while fixing the number of selected anchors at $K=4$. We use DeepSeek-R1-Distill-Qwen-1.5B and keep the retrieval configuration and gating thresholds unchanged. As shown in Table~\ref{tab:branch_budget}, increasing $B$ from 2 to 5 improves average NDCG@10 from 23.83 to 27.49, with gains on 10 of the 12 BRIGHT tasks. These results suggest that using five branches instead of two provides better counterfactual coverage for local credit estimation, potentially leading to more reliable credit assignment and improved retrieval performance.
\subsection{Preliminary Analysis of Retriever Adaptation}
\label{app:retriever_adaptation}

We investigate the effect of adapting the retriever
during RICE-PO optimization. To avoid the cost of
re-encoding the entire document corpus, we introduce
a lightweight linear adapter on query embeddings while
keeping the document encoder and corpus representations
fixed. All other experimental settings remain unchanged.
\begin{table*}[t]
\centering
\caption{Effect of query embedding adaptation on
NDCG@10 ($\times 100$) across 12 BRIGHT tasks.}
\label{tab:retriever_adaptation}
\resizebox{\textwidth}{!}{
\begin{tabular}{lccccccccccccc}
\toprule
Method & Bio. & Earth. & Econ. & Psy. & Rob. &
Stack. & Sus. & Leet. & Pony & AoPS &
TheoQ. & TheoT. & Avg. \\
\midrule
Query Embedding Adaptation
& 22.94 & 21.87 & 15.84 & 17.83 & 19.00
& 18.88 & 20.39 & 29.33 & 4.70 & 2.51
& 30.32 & 33.81 & 19.78 \\

RICE-PO (Fixed Retriever)
& 40.17 & 43.79 & 18.35 & 32.13 & 17.47
& 22.61 & 22.54 & 34.13 & 10.79 & 10.77
& 40.06 & 37.00 & \textbf{27.49} \\
\bottomrule
\end{tabular}
}
\end{table*}
As shown in Table~\ref{tab:retriever_adaptation},
query embedding adaptation reduces average NDCG@10
from 27.49 to 19.78, suggesting that a lightweight
linear adapter may disrupt the alignment between
query and document representations. Moreover,
concurrent retriever updates may shift the residual
reward distribution, potentially requiring
recalibration of the stability threshold
$\tau_{\mathrm{res}}$. These findings highlight
the challenges of jointly adapting the retriever
and the policy under fixed credit propagation thresholds.
\begin{table*}[t]
\centering
\caption{Performance with multiple consecutive thinking
blocks across 12 BRIGHT tasks (NDCG@10).}
\label{tab:multi_think}
\resizebox{\textwidth}{!}{
\begin{tabular}{lccccccccccccc}
\toprule
Method & Bio. & Earth. & Econ. & Psy. & Rob. &
Stack. & Sus. & Leet. & Pony & AoPS &
TheoQ. & TheoT. & Avg. \\
\midrule
RICE-PO
& 40.17 & 43.79 & 18.35 & 32.13 & 17.47
& 22.61 & 22.54 & 34.13 & 10.79 & 10.77
& 40.06 & 37.00 & \textbf{27.49} \\

Multi-think RICE-PO
& 32.57 & 40.30 & 11.92 & 25.74 & 16.14
& 20.37 & 18.94 & 24.51 & 6.41 & 10.30
& 26.97 & 30.51 & 22.05 \\

Multi-think Tree-GRPO
& 31.76 & 34.60 & 16.19 & 25.22 & 14.80
& 16.98 & 16.26 & 25.02 & 6.30 & 8.61
& 24.78 & 27.20 & 20.64 \\
\bottomrule
\end{tabular}
}
\end{table*}
\subsection{Credit Propagation Across Multiple Thinking Blocks}
\label{app:multi_think}

We further evaluate RICE-PO in a setting with multiple
consecutive thinking blocks before retrieval. Using
DeepSeek-R1-Distill-Qwen-1.5B and Diver-Retriever-4B
with BM25, the policy either continues reasoning with
\texttt{<think>...</think><action>continue</action>}
or initiates retrieval with
\texttt{<think>...</think><action>summary</action>}
followed by a summary action. Intermediate control
tokens are masked from the policy loss.
As shown in Table~\ref{tab:multi_think}, multi-think
RICE-PO outperforms multi-think Tree-GRPO on 10 of
the 12 BRIGHT tasks, improving average NDCG@10
from 20.64 to 22.05. Although performance decreases
relative to standard RICE-PO (27.49), the results
demonstrate that our credit estimation mechanism
remains beneficial under delayed retrieval feedback.

For multiple consecutive thinking blocks, RICE-PO
does not recursively propagate credit through
intermediate latent actions. Instead, when the
influence and residual-stability conditions are
satisfied, the subsequent summary-level advantage
is directly assigned to all preceding thinking
blocks associated with that summary. This enables
credit assignment across longer sequences of
latent reasoning actions without recursive
backward propagation.

\subsection{Robustness Across Retrieval Configurations}
\label{app:retriever_robustness}

We evaluate the robustness of RICE-PO across different
retrieval configurations using DeepSeek-R1-Distill-Qwen-1.5B.
Specifically, we compare BM25 with Diver-Retriever-4B
combined with BM25, while keeping the policy backbone,
gating thresholds, and other training settings fixed.
\begin{table*}[t]
\centering
\caption{Performance of RICE-PO under different retrieval
configurations across 12 BRIGHT tasks
(NDCG@10).}
\label{tab:retriever_robustness}
\resizebox{\textwidth}{!}{
\begin{tabular}{lccccccccccccc}
\toprule
Retriever & Bio. & Earth. & Econ. & Psy. & Rob. &
Stack. & Sus. & Leet. & Pony & AoPS &
TheoQ. & TheoT. & Avg. \\
\midrule
BM25
& 26.49 & 31.83 & 15.59 & 21.07 & 14.10
& 19.20 & 17.20 & 21.06 & 9.58 & 7.16
& 16.17 & 13.82 & 17.77 \\

Diver-4B + BM25
& 40.17 & 43.79 & 18.35 & 32.13 & 17.47
& 22.61 & 22.54 & 34.13 & 10.79 & 10.77
& 40.06 & 37.00 & \textbf{27.49} \\
\bottomrule
\end{tabular}
}
\end{table*}
As shown in Table~\ref{tab:retriever_robustness},
RICE-PO achieves an average NDCG@10 of 17.77 with
BM25 and 27.49 with Diver-4B + BM25. The combined
retrieval configuration improves performance on all
12 BRIGHT tasks, yielding an average gain of 9.72
points. These results demonstrate the sensitivity of
retrieval performance to the underlying retriever
and provide evidence of RICE-PO's applicability
across different retrieval configurations.

\subsection{Sensitivity to Retrieval Reward Metrics}
\label{app:reward_metrics}

We investigate the sensitivity of RICE-PO to different
retrieval reward functions. Specifically, we repeat
training using MRR@10 instead of NDCG@10 as the reward,
while keeping the policy backbone, training data,
retrieval configuration, and optimization settings
unchanged. Both models are evaluated using NDCG@10
and MRR@10 across the 12 BRIGHT tasks.
\begin{table*}[!t]
\centering
\caption{Effect of different retrieval reward functions
on RICE-PO performance across 12 BRIGHT tasks.}
\label{tab:reward_metrics}
\resizebox{\textwidth}{!}{
\begin{tabular}{llccccccccccccc}
\toprule
Reward & Evaluation & Bio. & Earth. & Econ. & Psy. &
Rob. & Stack. & Sus. & Leet. & Pony & AoPS &
TheoQ. & TheoT. & Avg. \\
\midrule
MRR@10 & NDCG@10
& 36.51 & 38.53 & 17.82 & 26.56 & 16.63
& 23.20 & 19.60 & 26.26 & 8.71 & 9.47
& 30.52 & 37.45 & 24.27 \\

MRR@10 & MRR@10
& 43.48 & 48.76 & 18.47 & 28.71 & 18.53
& 26.77 & 24.50 & 31.53 & 18.24 & 17.09
& 30.53 & 41.25 & 28.99 \\
\midrule
NDCG@10 & NDCG@10
& 40.17 & 43.79 & 18.35 & 32.13 & 17.47
& 22.61 & 22.54 & 34.13 & 10.79 & 10.77
& 40.06 & 37.00 & \textbf{27.49} \\

NDCG@10 & MRR@10
& 45.31 & 55.92 & 21.10 & 35.49 & 21.10
& 26.67 & 27.44 & 33.22 & 19.92 & 20.40
& 32.71 & 38.36 & \textbf{31.47} \\
\bottomrule
\end{tabular}
}
\end{table*}
As shown in Table~\ref{tab:reward_metrics}, RICE-PO
supports both NDCG@10 and MRR@10 as training rewards.
However, NDCG@10-based training achieves stronger
performance, improving average NDCG@10 and MRR@10
by 3.22 and 2.48 points, respectively, compared with
MRR@10-based training. It also outperforms the
MRR@10-trained model on 10 of the 12 tasks under
each evaluation metric. These results suggest that
NDCG@10 provides a more effective credit-assignment
signal in our retrieval setting.
\begin{table*}[!t]
\centering
\caption{Sensitivity to gating thresholds across 12 BRIGHT
tasks (NDCG@10).}
\label{tab:gating_thresholds}
\resizebox{\textwidth}{!}{
\begin{tabular}{lccccccccccccc}
\toprule
Thresholds & Bio. & Earth. & Econ. & Psy. & Rob. &
Stack. & Sus. & Leet. & Pony & AoPS &
TheoQ. & TheoT. & Avg. \\
\midrule
$\tau_{\mathrm{var}}=0.05,\ \tau_{\mathrm{res}}=0.03$
& 40.17 & 43.79 & 18.35 & 32.13 & 17.47
& 22.61 & 22.54 & 34.13 & 10.79 & 10.77
& 40.06 & 37.00 & \textbf{27.49} \\

$\tau_{\mathrm{var}}=0.01,\ \tau_{\mathrm{res}}=0.10$
& 37.90 & 36.99 & 20.04 & 27.09 & 15.30
& 22.49 & 19.99 & 26.47 & 6.73 & 9.67
& 30.03 & 36.91 & 24.13 \\

$\tau_{\mathrm{var}}=0.03,\ \tau_{\mathrm{res}}=0.50$
& 34.16 & 36.66 & 17.97 & 30.79 & 17.41
& 21.45 & 23.08 & 30.46 & 7.96 & 8.23
& 32.40 & 36.98 & 24.80 \\
\bottomrule
\end{tabular}
}
\end{table*}
\subsection{Sensitivity to Gating Thresholds}
\label{app:gating_thresholds}

We examine the sensitivity of RICE-PO to the gating
thresholds $\tau_{\mathrm{var}}$ and $\tau_{\mathrm{res}}$.
The former specifies the minimum influence variance
required for propagating local credit, while the latter
controls the maximum allowable residual variance.
We select a global threshold pair using the validation
set and apply it consistently across all tasks.
As shown in Table~\ref{tab:gating_thresholds},
$(\tau_{\mathrm{var}},\tau_{\mathrm{res}})=(0.05,0.03)$
achieves the highest average NDCG@10 of 27.49,
outperforming the alternative configurations by
2.69--3.36 points. These results suggest that
requiring sufficient local influence while limiting
residual uncertainty helps prevent unreliable
summary-level credit from being propagated to
reasoning tokens.

\subsection{Additional Related Work}
\label{sec:additional-related-work}

\paragraph{Retrieval-Augmented and Interactive Reasoning.}
Retrieval-augmented generation (RAG) integrates external knowledge
with parametric language models to support knowledge-intensive
reasoning~\citep{lewis2020retrieval}. Recent studies have extended
this paradigm from static evidence augmentation toward adaptive
retrieval and iterative reasoning. DeepRAG~\citep{guan2026deeprag}
formulates retrieval decisions as a sequential decision process,
allowing language models to determine when additional evidence
is needed. RAG-GD~\citep{li2026context} interprets retrieval
augmentation as an in-context optimization process and introduces
context-conditioned updates to the generator's evidence-use
interface. RIVO~\citep{li2026rivo} further investigates adaptive
trajectory control in long-horizon Deep Research agents through
lightweight, state-conditioned logit interventions. These approaches
highlight the growing importance of adaptive evidence acquisition
and interaction control in retrieval-augmented reasoning.

\paragraph{Graph-Structured Knowledge Retrieval and Reasoning.}
Beyond unstructured documents, graph-structured knowledge
provides relational evidence for knowledge-intensive
retrieval and reasoning tasks.
Kiest~\citep{li2023understand} retrieves relevant entities
and attributes from knowledge graphs to support entity
state tracking, while CTLP~\citep{li2024condensed}
encodes multi-hop relational paths into condensed graph
representations for zero-shot link prediction.
RiTeK~\citep{huang2026ritek} introduces a benchmark
for complex reasoning over textual knowledge graphs
and systematically evaluates retrieval methods for
relational evidence acquisition.
GLEN-Bench~\citep{huang2026glen} further investigates
graph-language reasoning over structured knowledge.
EvolveRouter~\citep{huang2026evolverouter} leverages
graph-structured relationships for adaptive agent
routing and prompt optimization, while
CAGE-CAL~\citep{huang2026counterfactual} models
inter-agent dependencies through counterfactual graphs
to improve reasoning reliability.
Together, these studies highlight the role of
graph-structured information in evidence retrieval
and its broader utility for supporting complex
reasoning across diverse tasks.

\paragraph{Retrieval in Other Applications.}
Retrieval-based knowledge integration has also been explored
across diverse downstream tasks. RT~\citep{li2024rt}
combines retrieval-based demonstration selection with
chain-of-thought reasoning for few-shot named entity recognition.
BiomedRAG~\citep{li2025biomedrag} integrates retrieved biomedical knowledge into language models to enhance healthcare prediction, while CancerLLM~\citep{li2026cancerllm} leverages retrieval-based example selection to support cancer phenotyping and diagnosis generation.
The robustness and self-awareness of retrieval-augmented
language models under noisy and heterogeneous evidence have
also been investigated~\citep{li2025benchmarking}.
Beyond conventional external retrieval, K2K~\citep{li2026efficient} replaces external
retrieval with internal key-based knowledge access, using
activation-guided probing and cross-attentive reranking
to integrate task-relevant knowledge.

Collectively, these studies advance evidence retrieval,
knowledge integration, structured reasoning, and adaptive
interaction across different settings. However, an important
challenge in training interactive retrieval agents is how to
attribute retrieval outcomes to the latent reasoning steps
that produce executable retrieval actions. Unlike directly
observable queries or summaries, latent reasoning cannot
be independently evaluated by the retriever, making
trajectory-level reward assignment potentially unreliable.
RICE-PO try to addresse this challenge by using executable
retrieval actions as intermediate credit anchors and
combining uncertainty-guided selection, local counterfactual
evaluation, and gated credit propagation to obtain
localized learning signals for reasoning-based retrieval agents.

\end{document}

%% file: 1_abs.tex
\begin{abstract}
Retrieval is increasingly moving from one-shot matching toward interactive reasoning, where language agents iteratively inspect evidence, reformulate queries, and search again. 
Training such agents raises a credit-assignment challenge: executable actions such as queries or summaries can be directly evaluated by the retriever, while latent reasoning steps are not directly observable and only affect future executable actions. 
This asymmetry makes outcome-level reward assignment unreliable, as the same final reward may credit reasoning steps that did not actually shape retrieval success. 
We propose RICE-PO, a critic-free policy optimization framework that converts retrieval interactions into localized learning signals. 
RICE-PO selects high-uncertainty executable actions as anchors, evaluates local counterfactual branches using retrieval metrics, and propagates credit to latent reasoning steps only when reasoning-to-action influence is strong and future residual effects are stable. 
On BRIGHT and BEIR, RICE-PO consistently outperforms prompt-based agents and group-based RL baselines under the same retriever setting. 
These results show that the structure of agent-environment interaction itself can provide useful supervision for training reasoning-based retrieval agents.
\end{abstract}

%% file: 2_introduction.tex
\section{Introduction}

As the scale and complexity of information needs continue to grow, retrieval is increasingly shifting from a one-shot matching problem toward an interactive reasoning process~\cite{long2025diver,jin2025search,li2025search}. Rather than relying on a single query-document comparison, a reasoning agent may need to interact with an external retriever over multiple rounds.  At each round, the agent reasons over the retrieved evidence to identify missing constraints, reformulates the query accordingly, and sends the refined query back to the retriever, which executes the search over the document collection. This interaction loop is especially important in reasoning-intensive retrieval, where relevant evidence is often not connected to the original query by surface lexical overlap or shallow semantic similarity.

Recent methods have therefore begun to incorporate retrieval interaction into reasoning-intensive retrieval. For example, ThinkQE~\citep{lei2025thinkqe} and Diver~\citep{long2025diver} use prompt-based interaction with retrieval tools to iteratively refine search inputs, but they do not learn from the interaction process itself. Instead, they rely on fixed prompting strategies and tool feedback at inference time. More recently, RL-based methods directly optimize query generation with retrieval feedback~\citep{li2026think, jiang2025deepretrieval}. While these methods improve retrieval-oriented generation, they still leave a key question unresolved: how should credit be assigned when retrieval trajectories contain actions with different observability and only sparse reward signals?


Our key observation is that retrieval interactions contain not only final outcome rewards, but also intermediate signals that can be used for credit assignment. Executable actions, such as generated queries or summaries, can be directly submitted to the retriever and evaluated with retrieval metrics. In contrast, latent reasoning actions are not directly executable or measurable; they should receive credit only when alternative reasoning continuations from the same history lead to meaningfully different executable summaries and retrieval rewards.
 This observation motivates a two-stage process for converting interaction into learning signals. 
First, the agent should identify which intermediate decisions are worth inspecting, since not every reasoning step substantially changes the following executable summary or its retrieval score.
Second, once an executable summary receives retrieval feedback, the agent should decide whether this feedback can be reliably attributed back to the preceding latent reasoning span.

Based on this view, we propose a critic-free framework that converts retrieval interactions into action-level learning signals. 
\textbf{First}, it selects critical decision points using three branch diagnostics: policy uncertainty, which identifies high-entropy executable summaries; reasoning-to-summary influence, which measures whether alternative reasoning continuations lead to different summary rewards; and future residual stability, which checks whether later retrieval rounds preserve rather than overturn the local reward difference.
\textbf{Second}, for the selected points, it constructs local counterfactual branches from the same interaction history and evaluates the resulting summaries with retrieval metrics, producing step-specific summary rewards. 
\textbf{Third}, it propagates these summary rewards to latent reasoning spans only when the future residual effect is stable and the reasoning-to-summary influence is sufficiently strong, preventing local rewards from being copied backward when later interaction steps dominate or overturn the current summary reward.
This design preserves the efficiency of group-based policy optimization while introducing fine-grained supervision without learned critics, process reward models, or exhaustive per-step rollouts. Our contributions are summarized as follows:

\begin{itemize}
    \item \textbf{We introduce an observability-aware credit assignment perspective} for reasoning-based retrieval, showing that executable summaries can provide direct retrieval feedback, while latent reasoning spans require selective backward credit propagation.

    \item \textbf{We propose a critic-free credit assignment framework} that exploits the executable-latent action asymmetry of retrieval agents: executable summaries provide step-specific retrieval rewards, while latent reasoning spans receive credit through uncertainty-triggered counterfactual evaluation and residual-stability gating.

    \item \textbf{We conduct experiments on BRIGHT and BEIR retrieval tasks}, demonstrating consistent gains over prompt-based agents and group-based RL baselines, together with analyses of uncertainty selection, local feedback quality, credit propagation reliability, and branching cost.
\end{itemize}

%% file: 2_related_work.tex
\section{Related Work} \paragraph{Interactive Retrieval for Reasoning-Intensive Tasks} Reasoning-intensive retrieval has long been a central problem in information retrieval (IR)~\cite{singhal2001modern}, with broad applications in search, recommendation\cite{ko2022survey,huang2026glen,wen2026smartsearch}, and question answering~\cite{li2022semantic,huang2026evolverouter}. Unlike conventional keyword-based retrieval, these tasks often require systems to interpret implicit constraints, infer missing context, and transform complex information needs into effective search queries.To improve interaction with the retriever, supervised methods fine-tune LLMs or sequence-to-sequence models to produce clearer and more specific queries~\citep{wu2021conqrr,liu2021conversational,mo2023convgqr}. In parallel, prompt-based methods use LLMs to generate pseudo-documents, reasoning traces, or expanded queries without additional training, including Query2Doc~\citep{wang2023query2doc}, ThinkQE~\citep{lei2025thinkqe}, Diver~\citep{long2025diver}, and TongSearch~\citep{qin2025tongsearch}. Despite their effectiveness, these methods either rely on costly rewrite annotations or depend on fixed prompting strategies, and thus provide limited guidance on how an agent should learn from the retrieval interaction itself. More recently, RL-based methods have been introduced to optimize retrieval-oriented generation using feedback from the retriever, moving beyond fixed prompts and supervised rewrite pairs. \paragraph{RL-based Agents for Interactive Retrieval} RL-based interactive retrieval is appealing because retrieval metrics can directly serve as rewards for generated queries. In this setting, an LLM acts as a policy that produces expanded queries, and the downstream retriever provides feedback through metrics such as NDCG or recall. Although PPO~\citep{schulman2017proximal} is commonly used for RL-based LLM optimization, its learned critic is costly to train. Value-free group-based methods, such as GRPO~\citep{shao2024deepseekmath}, reduce this cost by estimating advantages from multiple candidate expansions sampled for the same query. 
However, existing RL-based methods mainly optimize the final outcome of the retrieval interaction, with limited study in reasoning-intensive IR. Since such tasks require agents to interact with the retriever over multiple steps to uncover implicit constraints and missing evidence, outcome-level rewards provide little guidance on which internal actions in the interaction are actually useful. \paragraph{Reward Shaping and Process Supervision} To address reward sparsity, recent studies have proposed credit assignment methods for LLM-based RL~\citep{feng2025group,he2026hierarchy,ji2026tree}. For example, GIGPO~\citep{feng2025group} uses anchor states to group trajectories and improve advantage estimation. HGPO~\citep{he2026hierarchy} further estimates step-wise advantages by hierarchically grouping actions that share similar states and historical contexts, and adaptively weighting these groups to reduce variance and bias. Tree-GRPO~\citep{ji2026tree} constructs tree-structured agent trajectories to derive step-level supervision from outcome rewards, estimating relative advantages at both intra-tree and inter-tree levels for better credit assignment than chain-based RL. Unlike these methods, our approach treats retrieval as a structured interaction between latent reasoning actions and executable retrieval actions. It remains fully critic-free and does not require auxiliary value models, external judges, or human annotations, while deriving finer-grained credit signals from the retriever feedback available within the interaction itself.

%% file: 3_preliminaries.tex
\section{Preliminaries}

\paragraph{Problem setup}
We consider a query-side reasoning setting for reasoning-based retrieval, where an LLM agent interacts with a fixed retriever over multiple retrieval rounds. Given an initial user query $q_0$ and an initial set of retrieved documents $D_0$, the input is denoted as $x=(q_0,D_0)$. At each discrete step $t=1,2,\ldots,T$, the agent observes the original query $q_0$, the latest retrieved documents $D_{t-1}$, and the interaction history $h_{t-1}$. It then generates two textual segments: a reasoning segment $z_t \in \mathcal{V}^{n_z}$ and a summary segment $s_t \in \mathcal{V}^{n_s}$, where $\mathcal{V}$ denotes the token vocabulary. The summary $s_t$ is used directly as the query for the next retrieval round, and the fixed retriever returns a new document set
$
D_t = \mathcal{R}(s_t).
$
A full episode consists of a trajectory
$
\tau = \{(z_1,s_1,D_1), (z_2,s_2,D_2), \ldots, (z_T,s_T,D_T)\}.
$
The agent's behavior is governed by an LLM policy
$
\pi_\theta(z_t,s_t \mid q_0,D_{t-1},h_{t-1}),
$
parameterized by $\theta$, which defines a distribution over the generated reasoning and summary segments conditioned on the current retrieval context. After the final round, the environment returns a scalar retrieval reward $R(\tau)\in\mathbb{R}$, such as NDCG@10, based on the final retrieved documents or final summary. 
The summary can receive a direct retrieval reward because it is used as the query for the next retrieval round. In contrast, the reasoning span is internal: it only affects retrieval indirectly through the following summary and later turns. Therefore, directly assigning the summary reward to the reasoning tokens can cause credit misassignment.


%% file: 3_method.tex
\section{Method}

\begin{figure*}[t]
        \centering
        \includegraphics[width=1\columnwidth]{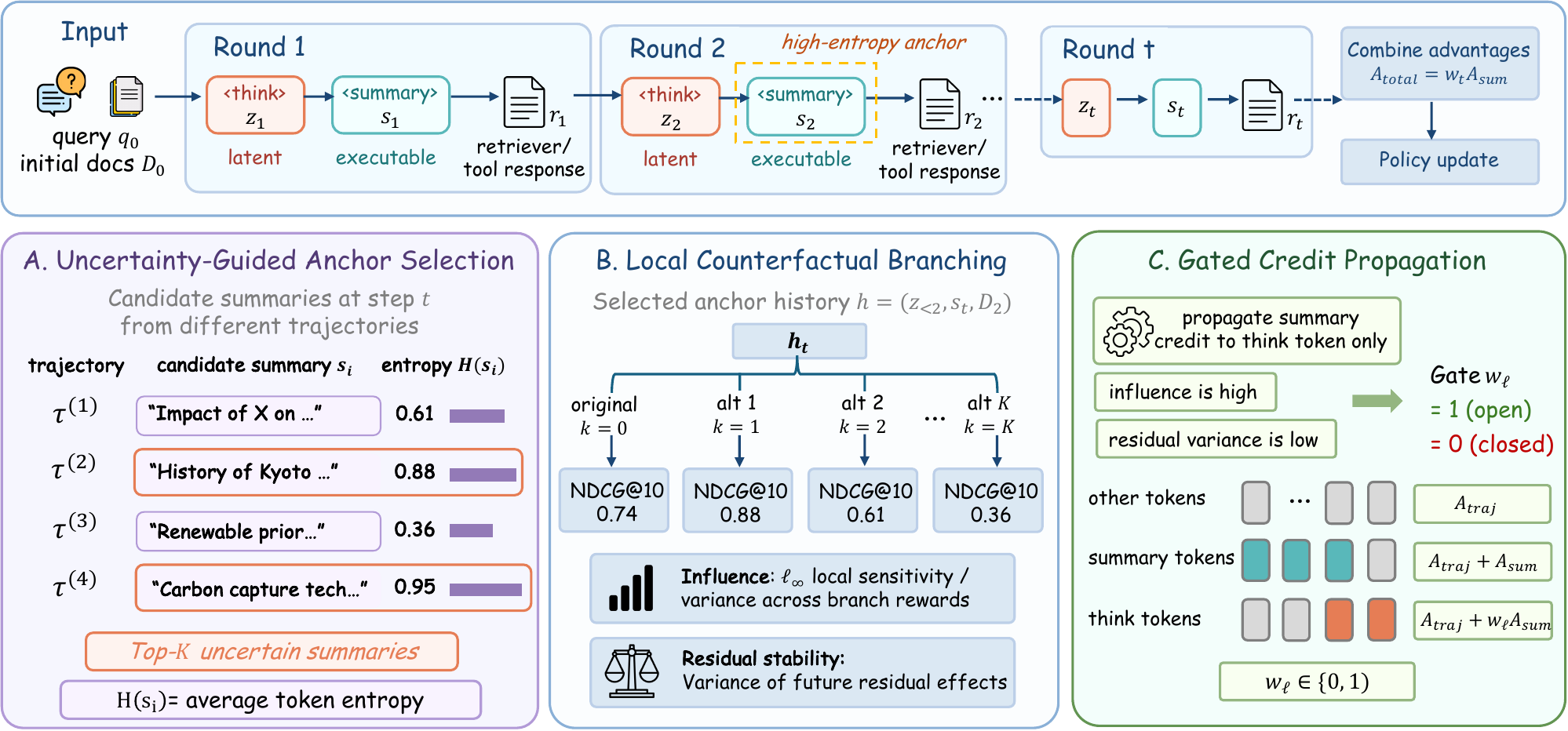}
	\caption{Overview of RICE-PO. The agent generates latent reasoning spans and executable summaries across multiple retrieval rounds. RICE-PO selects high-entropy summary anchors, estimates local influence and residual stability through counterfactual branches, and propagates summary-level credit to reasoning tokens only when the signal is reliable. }
	\label{fig:method}
  \vspace{-10pt}
\end{figure*}

\subsection{The Issue of Reasoning-Credit Ambiguity}
Group-relative advantage estimation typically compares trajectories sampled from the same query
and assigns advantages at the trajectory level. However, a multi-turn retrieval agent produces
heterogeneous actions within the interaction. Executable retrieval actions, such as summaries or
queries, can be directly submitted to the retriever and evaluated by metrics such as NDCG@10.
In contrast, latent reasoning actions are not directly executable; they affect retrieval only by shaping
subsequent executable actions and downstream interaction steps. As a result, broadcasting the final
trajectory reward to all reasoning tokens can introduce biased credit, because a successful final
retrieval outcome may be caused by later recovery rather than by the current reasoning action.

Rather than relying on step-level annotations or an external process reward model, we use executable
retrieval actions as intermediate credit anchors. Since summaries can be directly evaluated by the
retriever, they provide localized retrieval feedback for the surrounding interaction. The remaining
question is whether this localized feedback can be reliably attributed to the preceding latent reasoning
span. To answer this, we estimate three interaction-derived signals: policy entropy, reasoning-to-summary
influence, and residual stability. Policy entropy identifies uncertain decision points worth inspecting;
reasoning-to-summary influence measures whether the latent reasoning action actually shapes the
neighboring executable summary; and residual stability tests whether the local feedback remains
reliable after later interaction steps. Motivated by this view, we propose \textbf{RICE-PO},
a \textbf{Retrieval-Interaction Credit Estimation for Policy Optimization} framework that refines
trajectory-level advantages into localized advantages for both executable retrieval actions and latent
reasoning actions.

\subsection{Retrieval-Interaction Credit Estimation for Policy Optimization}
In this subsection, we introduce \textbf{RICE-PO}, as illustrated in Figure~\ref{fig:method}. 
Given an input query and initial retrieved documents, we first let the LLM generate \(N\) trajectories. Each trajectory is a multi-step reasoning--retrieval interaction, where each step contains a latent reasoning span \texttt{<think>...</think>}, an executable summary action \texttt{<summary>...</summary>}, and a retriever observation \texttt{<tool\_response>...</tool\_response>}.
We then identify critical decision points using policy uncertainty. For each selected anchor point, we perform Monte Carlo local sensitivity estimation by constructing local branches from the same interaction history. These branches are used to estimate two interaction-derived signals: the influence of the reasoning action on the following summary action, and the residual effect of the selected reasoning action on the final retrieval reward.
Based on these signals, we apply gated propagation to determine whether localized summary feedback should be assigned to the corresponding reasoning action. This produces step-level advantages for both executable summary actions and latent reasoning actions. Finally, we combine the trajectory-level advantage, summary-level advantage, and gated reasoning-level advantage into a final token-level advantage, which is used for policy optimization.

\subsubsection{Uncertainty-Guided Anchor Selection}
We begin by introducing entropy-based action selection, which organizes local counterfactual estimation around uncertain summary actions. The key intuition is that the advantage of a local action should be estimated more carefully when the policy is uncertain about that action. In retrieval agents, such uncertainty is particularly important for summary actions, since summaries are executable queries that directly determine the next retrieved evidence. Specifically, we first pool summary actions from trajectories sampled for the same query, and then select the top-$K$ actions according to their average token entropy. This uncertainty-triggered selection brings two main benefits: (i) it focuses local credit estimation on actions with ambiguous outcomes, and (ii) it reduces the cost of counterfactual branching.

Formally, let the $i$-th trajectory be
$\tau_i=\{(z_1^{(i)}, s_1^{(i)}), (z_2^{(i)}, s_2^{(i)}), \ldots, (z_T^{(i)}, s_T^{(i)})\}.$
For each summary action $s_t^{(i)}$, we define its entropy score as
\begin{equation}
    H(s_t^{(i)})
    =
    \frac{1}{|s_t^{(i)}|}
    \sum_{u \in s_t^{(i)}}
    H\!\left(\pi_\theta(\cdot \mid h_u)\right),
\end{equation}
where $h_u$ denotes the token-level history before token $u$. Based on this score, we define the selected action set as
\begin{equation}
    \mathcal{S}_{K}
    =
    \operatorname{TopK}_{(i,t)}
    H(s_t^{(i)}).
\end{equation}
The resulting set $\mathcal{S}_{K}$ contains the high-uncertainty summary actions on which we perform local counterfactual credit estimation.

\subsubsection{Local Counterfactual Branching} \emph{Influence.}
Given a selected summary action \(s_t^{(i)}\), we estimate whether its immediately preceding reasoning span \(z_t^{(i)}\), generated from history \(h_t^{(i)}\), influences the selected summary.
Ideally, we would compute the conditional expectation and variance of the local summary reward under the policy:
\begin{equation}
    \mu_t
    =
    \mathbb{E}_{z_t \sim \pi(\cdot \mid h_t)}
    \left[
        r(s_t(z_t))
    \right],
    \qquad
    \sigma_t^2
    =
    \mathrm{Var}_{z_t \sim \pi(\cdot \mid h_t)}
    \left[
        r(s_t(z_t))
    \right],
\end{equation}
where $s_t(z_t)$ denotes the summary induced by reasoning action $z_t$ and $r(s_t)$ is the corresponding retrieval reward.
Since the exact expectation is intractable, we approximate it with Monte Carlo branch samples.
At a selected trigger point \(t\), we fix the interaction history \(h_t\) before the reasoning action and construct \(B\) additional  branches from the same history. The original branch is denoted as \(b=0\). For each branch \(b=0,\ldots,B\), the LLM generates a reasoning span \(z_t^{(b)}\), its neighboring summary \(s_t^{(b)}\), and the corresponding retriever response. The neighboring summary then receives a local retrieval reward
$r_t^{(b)} = r(s_t^{(b)})$
which measures the immediate retrieval quality induced by the reasoning action at step \(t\).
In this paper, \(r(\cdot)\) is computed by using the generated summary together with the input query to retrieve documents with a fixed retriever, and measuring retrieval quality using NDCG@10.
The Monte Carlo estimates are then
\begin{equation}
    \widehat{\mu}_t
    =
    \frac{1}{B+1}
    \sum_{b=0}^{B}
    r_t^{(b)},
    \qquad
    \widehat{\sigma}_t^2
    =
    \frac{1}{B+1}
    \sum_{b=0}^{B}
    \left(
        r_t^{(b)}
        -
        \widehat{\mu}_t
    \right)^2 .
\end{equation}
Here, $\widehat{\sigma}_t^2$ serves as a Monte Carlo proxy for local sensitivity: a larger value means that different reasoning choices under the same history induce summaries with substantially different retrieval rewards.

\emph{Residual effect.}
To check whether the local summary signal is stable against future steps, we continue each branch to the maximum depth and compute the final reward $R_T^{(b)}$ from the last summary. For each branch, we define
\begin{equation}
    \epsilon_t^{(b)}
    =
    R_T^{(b)}
    -
    r_t^{(b)},
    \qquad
    \bar\epsilon_t
    =
    \frac{1}{B+1}
    \sum_{b=0}^{B}
    \epsilon_t^{(b)} ,
        \qquad
         \mathrm{ResVar}_t
    =
    \frac{1}{B+1}
    \sum_{b=0}^{B}
    \left(
        \epsilon_t^{(b)}
        -
        \bar\epsilon_t
    \right)^2 .
\end{equation}
A small \(\mathrm{ResVar}_t\) indicates that subsequent interaction steps have a similar residual effect across local branches. 
In this case, later retrieval rounds do not substantially reorder or overturn the local reward differences induced by the selected reasoning--summary pair. 
Thus, the local summary reward provides a more reliable proxy for the contribution of the paired reasoning action, making backward credit propagation less susceptible to later recovery or downstream correction.
\subsubsection{Gated Credit Propagation}
Based on these statistics, we propagate summary credit to the paired reasoning span only when the local reasoning choice is influential and the residual future effect is stable:
\begin{equation}
    \hat{\sigma}_t^2 \ge \tau_{\mathrm{var}}
    \quad \text{and} \quad
    \mathrm{ResVar}_t \le \tau_{\mathrm{res}} .
\end{equation}
The propagated reasoning credit is then
\begin{equation}
    \Delta_t^{\mathrm{think}}
    =
    \Delta_t^{\mathrm{sum}}.
\end{equation}
Importantly, these Monte Carlo statistics are not used as independent rewards; they only determine how strongly the summary-level credit should be copied back to the paired reasoning span.

\paragraph{Final token-level advantage and policy optimization.}
In this section, we incorporate step-level credit signals into the token-level advantages used for policy optimization. 
We start from the group-normalized trajectory advantage $A_i$, computed from the final-step retrieval reward, i.e., the NDCG@10 obtained by the last summary in trajectory $\tau_i$. This trajectory-level advantage is broadcast to all generated tokens in $\tau_i$. For a selected summary span, we assign the local summary advantage $A_t^{\mathrm{sum}}$, computed from the NDCG@10 scores of alternative summaries generated from the same prefix at step $t$. 
 For tokens in the paired reasoning span, we assign the reasoning-level advantage $A_t^{\mathrm{think}}$, obtained by gated propagation from the corresponding summary advantage. Tokens outside the selected spans retain the original trajectory-level advantage. Formally, for token $u$ in trajectory $\tau_i$,
\begin{equation}
A_u^{\mathrm{final}}
=
\begin{cases}
A_t^{\mathrm{sum}},
& u \in [l_s^{(i,t)}, r_s^{(i,t)}), \\[4pt]
A_t^{\mathrm{think}},
& u \in [l_z^{(i,t)}, r_z^{(i,t)}), \\[4pt]
A_i^{},
& \text{otherwise}.
\end{cases}
\end{equation}

\begin{equation}
A_t^{\mathrm{think}}
=
\begin{cases}
A_t^{\mathrm{sum}},
&
\hat{\sigma}_t^2 \ge \tau_{\mathrm{var}}
\ \text{and}\
\mathrm{ResVar}_t \le \tau_{\mathrm{res}},
\\[4pt]
A_i,
&
\text{otherwise}.
\end{cases}
\end{equation}

Given $A_u^{\mathrm{final}}$, we optimize the policy with a PPO-style clipped objective. Let
\begin{equation}
\rho_u(\theta)
=
\frac{\pi_{\theta}(a_u\mid h_u)}
{\pi_{\theta_{\mathrm{old}}}(a_u\mid h_u)}
\label{eq:token-ratio}
\end{equation}
be the token-level importance sampling ratio. The final objective is
\begin{equation}
\mathcal{J}(\theta)
=
\mathbb{E}
\left[
\sum_{u}
\min
\left(
\rho_u(\theta)A_u^{\mathrm{final}},
\mathrm{clip}\left(\rho_u(\theta),1-\epsilon,1+\epsilon\right)
A_u^{\mathrm{final}}
\right)
\right]
-
\beta_{\mathrm{KL}}
D_{\mathrm{KL}}
\left(
\pi_{\theta}(\cdot\mid x)
\|
\pi_{\mathrm{ref}}(\cdot\mid x)
\right)
\label{eq:ppo-objective}
\end{equation}
, where $\epsilon$ is the clipping coefficient and $\beta_{\mathrm{KL}}$ controls the strength of the KL penalty.

\begin{proposition}[Summary-to-Reasoning Credit Approximation]
\label{prop:summary_credit}
Let $z_t$ be the reasoning action under history $h_t$ that induces summary $s_t$.
Assume $R=r(s_t)+\epsilon_t$, where $r(s_t)$ is the summary retrieval reward and $\epsilon_t$ is the remaining future effect. Then
\[
A_t^{\mathrm{think}}
=
A_t^{\mathrm{sum}}+\Delta_\epsilon,
\]
where
\[
A_t^{\mathrm{sum}}
=
r(s_t)
-
\mathbb{E}_{z_t'\sim\pi(\cdot|h_t)}[r(s_t')],
\quad
\Delta_\epsilon
=
\mathbb{E}[\epsilon_t|h_t,z_t]
-
\mathbb{E}_{z_t'\sim\pi(\cdot|h_t)}
[\epsilon_t'|h_t,z_t'] .
\]
Thus, if $|\Delta_\epsilon|\leq\delta$,
\[
|A_t^{\mathrm{think}}-A_t^{\mathrm{sum}}|\leq\delta .
\]
If additionally $|A_t^{\mathrm{sum}}|\geq\tau$, the summary signal is non-trivial and can be propagated to the reasoning tokens of $z_t$.
\end{proposition}

%% file: 4_experiment.tex
\section{Experiments}

\subsection{Experiment Setup} 
\paragraph{Benchmarks}We first evaluate our method on the BRIGHT benchmark~\cite{su2024bright}, which comprises 12 reasoning-intensive retrieval tasks. BRIGHT is designed to evaluate the reasoning ability of retrieval systems, where relevant documents often require in-depth reasoning to identify their connection to complex queries. The benchmark includes queries from diverse sources, such as StackExchange, LeetCode, and TheoremQA, and covers domains including biology, economics, programming, and mathematics. In addition to BRIGHT, we evaluate our method on BEIR~\cite{thakur2021beir}, a broader benchmark for general information retrieval across diverse domains and task types. Specifically, we use five BEIR datasets: DBPedia-Entity, FiQA-2018, SciFact, Touché-2020, and TREC-COVID. These datasets cover entity retrieval, finance question answering, scientific fact checking, argument retrieval, and biomedical information retrieval. We use NDCG@10 as the evaluation metric.

\begin{table*}[t]
\centering
\renewcommand\arraystretch{1.1}
\caption{Performance on BRIGHT. All RL-based methods use the same fixed retriever, Diver-4B+BM25~\cite{long2025diver}. We report NDCG@10 for each task and the macro average across tasks.}
\resizebox{\textwidth}{!}{
\begin{tabular}{l|ccccccc|cc|ccc|c}
\hline
Method
& \multicolumn{7}{c|}{StackExchange}
& \multicolumn{2}{c|}{Coding}
& \multicolumn{3}{c|}{Theorem-based}
& Avg. \\
& Bio. & Earth. & Econ. & Psy. & Rob. & Stack. & Sus.
& Leet. & Pony
& AoPS & TheoQ. & TheoT.
& \\
\hline
\rowcolor{gray!20}
\multicolumn{14}{c}{\textit{Existing retrieval  baselines}} \\
RaDeR
& 36.10 & 42.90 & 25.20 & 37.90 & 16.60 & 27.40 & 25.00
& 34.80 & 11.90
& 12.00 & 37.70 & 43.40
& 29.20 \\
DeepRetrieval
& 22.48 & 25.60 & 17.06 & 21.13 & 13.46 & 16.75 & 16.93
& 24.99 & 8.13
& 3.88 & 27.90 & 22.01
& 18.36 \\
TongSearch-QR
& 46.20 & 45.10 & 31.20 & 39.60 & 25.30& 28.70 & 28.40
& 31.20 & 16.30
& 10.80 & 40.00& 39.30
& 31.90 \\
Diver
& 60.00 & 55.90 & 31.80 & 47.90 & 27.10 & 33.90 & 31.90
& 35.10 & 23.10
& 16.80 & 36.90 & 46.60
& 37.20 \\
\hline
\rowcolor{gray!20}
\multicolumn{14}{c}{\textit{DeepSeek-R1-Distill-Qwen-1.5B}} \\
Prompting agents
& 20.41 & 25.14 & 9.74 & 13.77 & 7.99 & 6.81 & 8.79
& 13.08 & 7.01
& 3.03 & 13.48 & 11.74
& 11.75 \\
GRPO
& 34.50 & 45.61 & 18.71 & 31.68 & 17.86 & 19.88 & 22.38 & 27.98 & 11.25 & 5.42 & 31.08 & 33.48 & 24.98 \\
GiGPO
& 37.00 & 41.58 & 23.41 & 29.10 & 16.81 & 22.28 & 21.41 & 33.53 & 11.83 & 7.96 & 31.14 & 34.46 & 25.88 \\
HGPO
& 36.36 & 41.73 & 20.02 & 26.94 & 15.00 & 23.62 & 20.46
& 29.16 & 11.65
& 7.81 & 34.58 & 38.01
& 25.45 \\
Tree-GRPO&
37.24 & 44.09 & 19.42 & 28.73 & 18.99 & 22.99 & 21.65 & 29.25 & 11.73 & 10.79 & 30.68 & 36.49&26.00\\
\rowcolor[RGB]{222,230,241}
\textbf{RICE-PO(ours)}
& 40.17 & 43.79 & 18.35 & 32.13 & 17.47 & 22.61 & 22.54 & 34.13 & 10.79 & 10.77 & 40.06 & 37.00 & 27.49  \\
\hline
\rowcolor{gray!20}
\multicolumn{14}{c}{\textit{Qwen3-4B-Thinking-2507}} \\
Prompting agents
& 54.36 & 54.89 & 27.04 & 40.54 & 26.12 & 32.44 & 29.38
& 30.80 & 15.36
& 12.99 & 40.38 & 47.27
& 34.30 \\
GRPO
& 58.03 & 56.82 & 29.08 & 42.57 & 27.97 & 34.52 & 27.28
& 36.61 & 17.44
& 12.44 & 44.92 & 49.29
& 36.42 \\
GiGPO
& 57.82 & 57.80 & 27.91 & 41.32 & 27.94 & 32.75 & 29.35
& 35.69 & 17.91
& 16.15 & 45.94 & 45.19 & 36.31 \\
HGPO&
61.77 & 59.53 & 27.87 & 41.93 & 26.70 & 33.99 & 32.72
& 30.81 & 19.72
& 16.07 & 43.02 & 45.16 & 36.61  \\
Tree-GRPO
& 60.19 & 56.29 & 27.91 & 43.47 & 27.60 & 31.09 & 30.09
& 35.91 & 16.92
& 15.59 & 43.12 & 48.47
& 36.39 \\
\rowcolor[RGB]{222,230,241}
\textbf{RICE-PO(ours)}
&59.32 & 59.15& 28.49 & 42.00 & 28.18 & 35.00 & 30.73
& 35.37 & 17.14
& 18.97 & 45.39 & 49.37 & 37.76 \\
\hline
\end{tabular}
}
\label{tab:reasonir-results}
\vspace{-10pt}
\end{table*}

\paragraph{Baselines}
For BRIGHT, we compare our approach with two groups of competitive baselines.
First, we include existing reasoning-based retrieval, including RaDeR~\cite{das2025rader}, DeepRetrieval~\cite{jiang2025deepretrieval}, TongSearch-QR~\cite{qin2025tongsearch}, and Diver~\cite{long2025diver}. Among them, Diver represents a strong state-of-the-art prompting-based baseline. In the interaction process, Diver uses a hybrid retriever that combines Diver-4B and BM25, and employs DeepSeek-R1-Distill-Qwen-14B to summarize the retrieved information.
Second, we compare with recent RL-based credit assignment methods, including GRPO~\citep{shao2024deepseekmath}, GiGPO~\citep{feng2025group}, HGPO~\citep{he2026hierarchy}, and Tree-GRPO~\citep{ji2025tree}. Since there is limited prior RL work specifically designed for reasoning-intensive retrieval, we adapt them to our retrieval-agent setting. For fair comparison, we evaluate them with two backbone language models, DeepSeek-R1-Distill-Qwen-1.5B and Qwen3-4B-Thinking-2507. In all prompting-agent settings, the LLM is used as the generator and Diver-4B+BM25 is used as the retriever, matching the retriever used by Diver and our method.
For BEIR, following Rank1~\citep{weller2025rank1}, we include several strong reranking baselines: MonoT5-3B~\citep{nogueira2019document}, mT5-13B fine-tuned on MMARCO for multilingual retrieval tasks~\citep{jeronymo2023neuralmind}, and RankLLaMA-7B/13B~\citep{ma2024fine}. For RL-based baselines, we compare with GRPO~\citep{shao2024deepseekmath}, GiGPO~\citep{feng2025group}, HGPO~\citep{he2026hierarchy}, and Tree-GRPO~\citep{ji2025tree}. 

\subsection{Performance on Bright}
Table~\ref{tab:reasonir-results} reports NDCG@10 on BRIGHT. Existing retrieval baselines are strong; in particular, the state-of-the-art baseline Diver uses a larger 14B model to generate retrieval queries and summarize retrieved documents. Since RL training with a 14B model is computationally expensive, we evaluate our method with more practical 1.5B and 4B backbones.
RL training substantially improves smaller models. With a 1.5B backbone, our method outperforms the non-interaction 3B baseline DeepRetrieval; with a 4B backbone, it further surpasses the 7B baseline RaDeR and achieves an average score of 37.76, exceeding the Diver baseline average of 37.20. This suggests that learning from interaction signals can be a competitive alternative to relying solely on larger retrieval-generation pipelines.
Under the 1.5B setting, GRPO improves the average score from 11.75 to 24.98 by optimizing the final-step retrieval reward. However, this trajectory-level reward does not distinguish which intermediate reasoning or summary action caused the gain. GiGPO and HGPO introduce step-level or history-aware credit assignment, but their local signals are still mainly redistributed from the final reward. Tree-GRPO uses branching over alternative rollouts, but its credit estimation still depends on full-trajectory outcomes.
In contrast, our method exploits the executable nature of summaries: intermediate summaries can be directly evaluated by NDCG@10 and used as local credit anchors. We then propagate this credit to the paired reasoning span only when the influence and residual-effect gates indicate that the signal is reliable. With the 1.5B backbone, this retrieval-specific credit assignment improves the average score to 27.49, outperforming GRPO, GiGPO, HGPO, and Tree-GRPO under the same fixed retriever. The same trend holds for the 4B backbone, where RICE-PO achieves the best average score of 37.76, improving over GRPO, GiGPO, HGPO, and Tree-GRPO by 1.34, 1.45, 1.15, and 1.37 points, respectively.

\subsection{Performance on BEIR}
\begin{wraptable}{r}{0.62\textwidth}
\vspace{-10mm}
\centering
\small
\renewcommand\arraystretch{1.05}
\caption{Performance on BEIR datasets: DBPedia-Entity, FiQA-2018, SciFact, Touché-2020, and TREC-COVID. We report NDCG@10 for each dataset and the macro average across tasks.}
\label{tab:beir-5subset}
\vspace{1.5mm}
\resizebox{\linewidth}{!}{
\begin{tabular}{l|ccccc|c}
\hline
Method 
& DBP & FiQA & SciFact & Touché & TREC-COVID & Avg. \\
\hline
\rowcolor{gray!20}
\multicolumn{7}{c}{\textit{Retrieval baselines}} \\
BM25S & 32.0 & 25.4 & 69.1 & 34.7 & 68.8 & 46.00 \\
MonoT5-3B & 44.5 & 46.5 & 76.1 & 30.7 & 79.6 & 55.48 \\
RankLLaMA-7B & 43.7 & 42.1 & 71.1 & 41.4 & 80.2 & 55.70 \\
RankLLaMA-13B & 44.9 & 44.1 & 72.7 & 39.2 & 80.8 & 56.34 \\
Rank1-7B & 38.9 & 39.5 & 77.2 & 22.8 & 81.9 & 52.06 \\
Rank1-14B & 37.4 & 37.9 & 77.0 & 27.1 & 78.2 & 51.52 \\
Rank1-32B & 40.7 & 41.8 & 76.8 & 19.9 & 81.9 & 52.22 \\
\hline
\rowcolor{gray!20}
\multicolumn{7}{c}{\textit{DeepSeek-R1-Distill-Qwen-1.5B}} \\
Prompting & 59.65 & 40.94 & 69.21 & 38.27 & 95.01 & 60.62 \\
GRPO & 62.35 & 43.49 & 72.41 & 42.04 & 97.15 & 63.49 \\
HGPO & 63.15 & 43.55 & 73.91 & 41.21 & 96.03 & 63.57 \\
GiGPO & 63.30 & 43.85 & 75.37 & 41.47 & 95.61 & 63.92 \\
Tree-GRPO & 64.81 & 45.61 & 74.65 & 40.18 & 97.49 & 64.55 \\
\rowcolor[RGB]{248,228,222}
\textbf{RICE-PO(ours)} & 65.11 & 46.10 & 76.90 & 41.85 & 98.33 & 65.66 \\
\hline
\rowcolor{gray!20}
\multicolumn{7}{c}{\textit{Qwen2.5-3B-Instruct}} \\
Prompting & 61.56 & 41.34 & 73.00 & 38.36 & 93.87 & 61.63 \\
GRPO & 64.57 & 45.16 & 76.06 & 39.07 & 98.15 & 64.60 \\
HGPO & 64.23 & 45.58 & 76.40 & 39.90 & 97.40 & 64.70 \\
GiGPO & 65.69 & 46.07 & 75.90 & 38.32 & 97.48 & 64.69 \\
Tree-GRPO & 64.71 & 45.65 & 76.47 & 40.65 & 97.62 & 65.02 \\
\rowcolor[RGB]{248,228,222}
\textbf{RICE-PO(ours)} & 66.29 & 45.94 & 76.96 & 40.63 & 98.59 & 65.68 \\
\hline
\end{tabular}
}
\vspace{-4mm}
\end{wraptable}

Table~\ref{tab:beir-5subset} reports results on BEIR, evaluating general retrieval performance under a fixed retriever. Compared with strong retriever and reranker baselines (e.g., MonoT5, RankLLaMA, Rank1), query-generation policies consistently achieve higher scores, showing the benefit of optimizing the query itself.
Among RL-based methods, our approach achieves the best performance across both backbones. With DeepSeek-R1-Distill-Qwen-1.5B, our method improves the average score to 65.66, outperforming GRPO, GiGPO, HGPO, and Tree-GRPO. Similar gains are observed with Qwen2.5-3B-Instruct, where our method reaches 65.68, again achieving the best average.
Compared with BRIGHT, the improvements on BEIR are smaller but consistent. This is expected, as BEIR tasks are less reasoning-intensive and rely more on surface-level matching. Nevertheless, our method still provides gains, suggesting that retrieval-specific credit assignment generalizes beyond multi-step reasoning settings. In particular, using executable summaries as intermediate credit signals remains beneficial even when the interaction depth is limited.



\subsection{Further Analysis}
\begin{wrapfigure}{r}{0.4\textwidth}
    \centering
    \vspace{-4mm}
    \includegraphics[width=\linewidth]{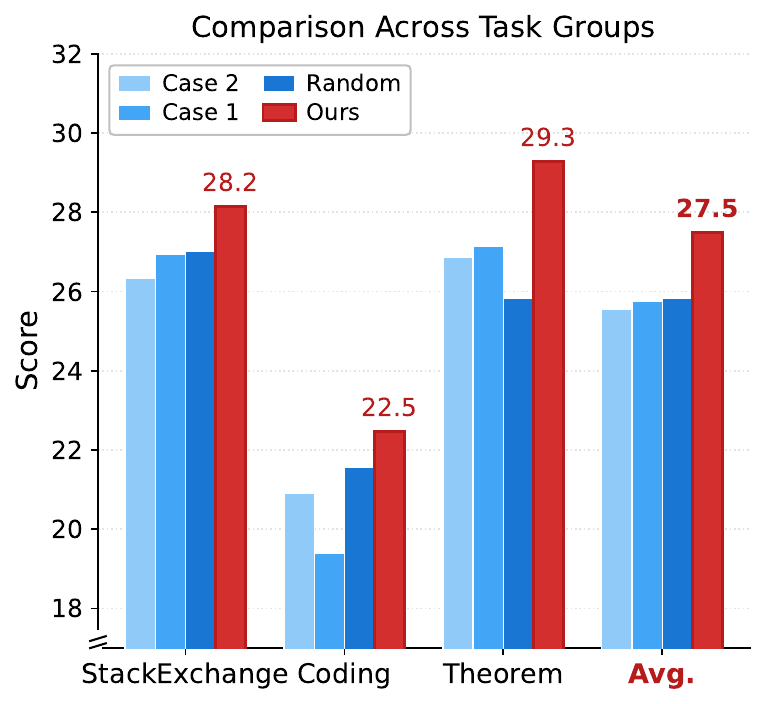}
    \vspace{-2mm}
    \caption{Ablation of reasoning credit assignment variants}
    \label{fig:case_ablation}
    \vspace{-4mm}
\end{wrapfigure}
\paragraph{Ablation of Credit Propagation Mechanisms} We conduct an ablation study on BRIGHT with DeepSeek-R1-Distill-Qwen-1.5B to analyze different reasoning-credit assignment strategies. 
For a fair comparison, all variants use the same entropy-triggered anchor points. 
Case 1 assigns the reasoning advantage from the paired step-level summary advantage. 
Case 2 assigns it from the final-summary reward, where NDCG@10 is computed from the last summary. 
Random uniformly selects between these two strategies. 
Our method instead uses influence and residual-effect estimates to dynamically decide whether summary-level credit should be propagated to the paired reasoning span.
As shown in Figure~\ref{fig:case_ablation}, our method achieves the best macro average, improving over Case 2, Case 1, and Random. 
The gains are particularly evident on coding and theorem-based tasks, where reasoning steps more directly affect the quality of subsequent retrieval summaries. 
In contrast, the gap is smaller on tasks with weaker reasoning–retrieval coupling, such as economics or Pony. 
These results indicate that neither always propagating local summary credit nor always relying on final-summary credit is sufficient. 
Instead, dynamically gating credit propagation based on interaction-derived signals provides a more reliable mechanism for assigning credit to latent reasoning actions.

\paragraph{Influence vs. Effect}A central design choice of our method is whether reasoning credit should be propagated based on local influence, downstream effect, or both. The key difference in Figure~\ref{fig:gate_ablation} lies in how credit propagation is gated.

The \textit{Influence Only} variant considers only the effect of the reasoning action on the induced summary, i.e., whether changing the reasoning step leads to different retrieval summaries. 
\begin{wrapfigure}{r}{0.35\textwidth} \centering \vspace{-4mm} \includegraphics[width=\linewidth]{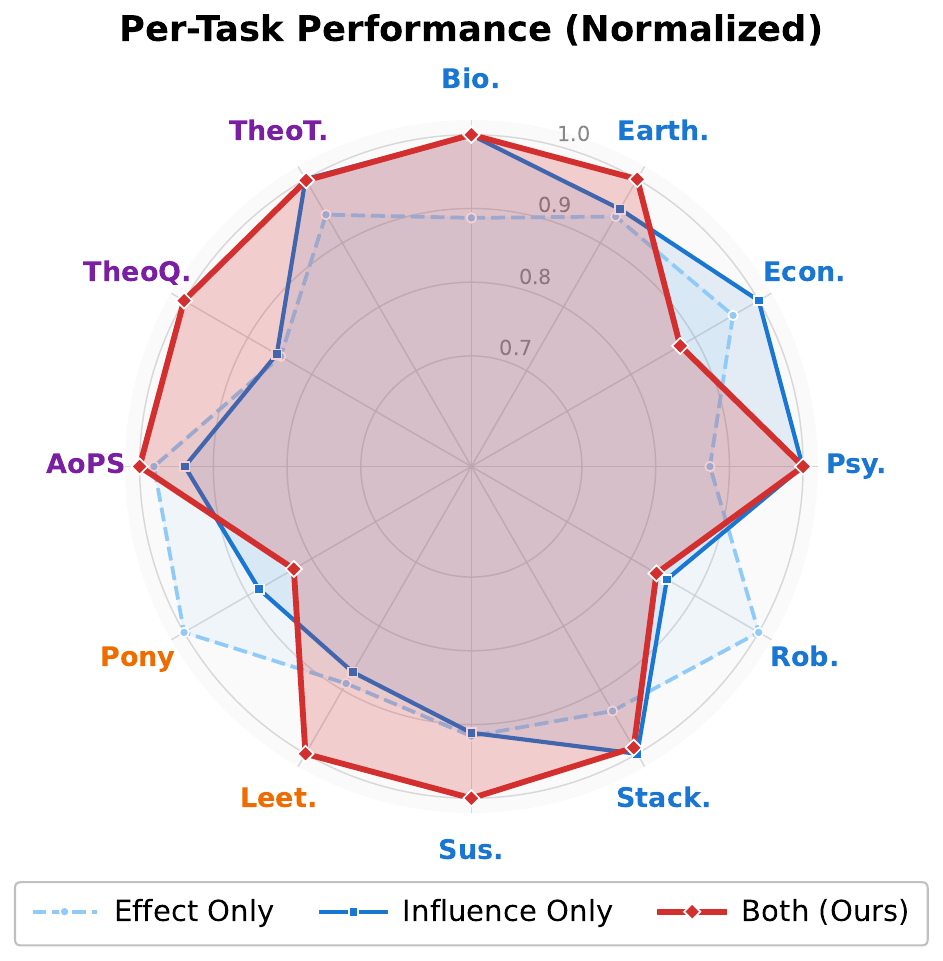} \vspace{-15pt} 
\caption{Model performance with effect-only and influence-only credit propagation} \label{fig:gate_ablation} \vspace{-15pt} 
\end{wrapfigure}
However, it ignores whether this change ultimately affects the final retrieval reward. 
In contrast, the \textit{Effect Only} variant considers only the impact of the reasoning action on the final reward, without explicitly measuring its influence on the intermediate summary.
Figure~\ref{fig:gate_ablation} shows that the two gates are complementary. 
\textit{Influence Only} is competitive on Economics and Psychology, where changing the intermediate summary is often useful, but it is weaker on Sustainable Living, LeetCode, and TheoremQA-Questions, where local changes may not survive later retrieval rounds. 
\textit{Effect Only} performs well on Robotics and Pony, but it can be noisy because final reward changes may come from later steps rather than the selected reasoning action. 
Our full method combines both signals and achieves stronger performance on most tasks, such as Biology, Earth Science, AoPS, and TheoremQA, supporting more reliable reasoning-level credit assignment.

\begin{wrapfigure}{r}{0.36\textwidth} \centering \vspace{-4mm} \includegraphics[width=\linewidth]{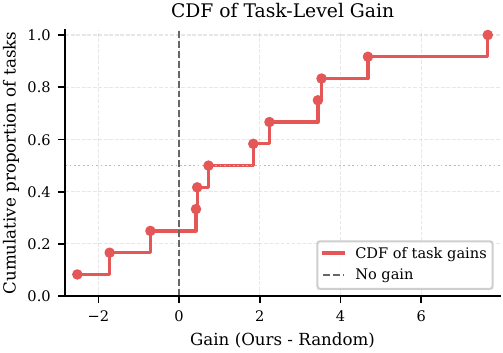} \vspace{-2mm} \caption{CDF of task-level gains from entropy-based triggering over random triggering.} \label{fig:entropy_ablation} \vspace{-4mm} 
\end{wrapfigure}
\paragraph{Entropy-based Triggering}
Figure~\ref{fig:entropy_ablation} shows the cumulative distribution of task-level gains from replacing random triggering with entropy-based triggering on BRIGHT. 
The x-axis reports the per-task gain in NDCG@10, computed as the difference between our method and the random-trigger baseline, while the y-axis shows the cumulative proportion of tasks. 
The vertical dashed line marks zero gain.
The curve shows that most tasks fall on the positive side, with 9 out of 12 tasks benefiting from entropy-based triggering. 
The median gain is about \(+1.28\), indicating that the improvement is not driven only by a single outlier. 
Several tasks obtain clear gains, such as TheoremQA-Questions \((+7.65)\) and Psychology \((+4.68)\), while only a few tasks, such as Economics and Pony, show negative gains. 
This suggests that entropy is an effective signal for selecting informative reasoning steps for local credit estimation, leading to broadly distributed improvements over random triggering.

%% file: 6_conclution.tex
\section{Conclusion}
We study how to better leverage interaction signals in reasoning-based retrieval agents. Existing RL methods mainly rely on trajectory-level rewards, which provide limited supervision for intermediate reasoning steps. We propose RICE-PO, a framework that directly exploits retrieval interaction signals by treating summaries as executable actions and using their retrieval outcomes as intermediate feedback. To ensure reliable learning, we introduce entropy-based triggering to select informative decision points and use influence and residual-effect estimation to determine when these signals should be propagated.
Experiments on BRIGHT and BEIR show that our method consistently outperforms prior RL approaches under the same retriever setting, with especially strong improvements for smaller models. These results suggest that effectively utilizing interaction signals, rather than relying solely on final rewards or larger models, is key to improving reasoning-intensive retrieval. More broadly, our work highlights the importance of designing learning algorithms that align with the structure of agent–environment interactions in multi-step language systems.

\section{Limitations and Future Directions}
\label{sec:limitations}

This work studies retrieval-interaction credit assignment with a fixed retriever and a bounded local branching budget. This setting lets us isolate how executable retrieval actions can provide reliable learning signals for latent reasoning steps while keeping training practical. Future work can extend this framework along three directions: combining interaction-derived credit with adaptive retrievers, scaling to larger backbones, and developing more efficient estimators for reasoning-to-summary influence and residual stability. More broadly, we view RICE-PO as an initial step toward training language agents from the structure of their interaction with external tools, beyond retrieval alone.

\section{Acknowledgments}
We sincerely thank Yifan Zhang from the University of Massachusetts Lowell for her assistance with GPU server. This work was supported in part by the National Institutes of Health (NIH; grants 5R01DA056470 and 5R01AG080670) and the U.S. Department of Veterans Affairs Health Systems Research (HSR; awards 5I01HX003711 and 1I01HX003969). The content is solely the responsibility of the authors and does not necessarily represent the official views of the NIH or the U.S. Department of Veterans Affairs.
